\documentclass[letter, 10pt, conference]{ieeeconf}  

\usepackage{graphicx}
\usepackage{epstopdf}

\usepackage{cite}

\usepackage{amsmath,amssymb,amsthm}
\usepackage{algorithm,algpseudocode}

\usepackage{color}

\usepackage{subfigure}

\usepackage{tikz}
\usetikzlibrary{shapes,arrows,automata,calc,trees,positioning,fit,shapes,calc,patterns,3d}
\usepackage{pgfplots}
\usepgfplotslibrary{fillbetween,groupplots}

\IEEEoverridecommandlockouts
\theoremstyle{plain}
\newtheorem{theorem}{Theorem}

\newtheorem{proposition}[theorem]{Proposition}

\theoremstyle{definition}

\theoremstyle{remark}

\DeclareMathOperator{\tr}{tr}

\DeclareMathOperator*{\argmin}{arg\,min}

\author{Miguel Calvo-Fullana and Jonathan P. How
\thanks{This work was supported in part by Lockheed Martin Corporation and by ONR under BRC award N000141712072. Authors are at the Massachusetts Institute of Technology, Cambridge, MA, USA. e-mail: \{cfullana,jhow\}@mit.edu.
}
}

\title{\LARGE \bf
Distributed Filtering with Value of Information Censoring 
}

\begin{document}

\maketitle

\begin{abstract}
This work presents a distributed estimation algorithm that efficiently uses the available communication resources. The approach is based on Bayesian filtering that is distributed across a network by using the logarithmic opinion pool operator. Communication efficiency is achieved by having only agents with high Value of Information (VoI) share their estimates, and the algorithm provides a tunable trade-off between communication resources and estimation error. Under linear-Gaussian models the algorithm takes the form of a censored distributed Information filter, which guarantees the consistency of agent estimates. Importantly, consistent estimates are shown to play a crucial role in enabling the large reductions in communication usage provided by the VoI censoring approach. We verify the performance of the proposed method via complex simulations in a dynamic network topology and by experimental validation over a real ad-hoc wireless communication network. The results show the validity of using the proposed method to drastically reduce the communication costs of distributed estimation tasks. 
\end{abstract}

\section{Introduction}

Modern multiagent robotic systems are often equipped with a diverse range of sensing capabilities that can be deployed to perform a range of tasks, such as monitoring, patrolling, and tracking. These tasks are typically conducted in complex and dynamic environments and achieving accurate estimates is crucial to task accomplishment. Due to the heterogeneity of the team and the nature of the environment, the agents must collaborate to obtain these accurate estimates, which puts an emphasis on intelligent communication decisions, especially for systems with networks that have limited capability, reliability and a time-varying topology.

The study of estimation techniques, distributed or not, has a long-standing tradition, with many algorithms being celebrated. Among those, the Kalman filter \cite{kalman1960new} is universally found in practical systems. Due to its ubiquity, several approaches to obtain distribute variants of the Kalman filter exist \cite{olfati2009kalman,olfati2007distributed,kamgarpour2008convergence,das2016consensus}, with the consensus Kalman filter \cite{olfati2009kalman} being its most famous exponent. Nonetheless, some issues plague these methods. For general networks, consensus among the agents is only guaranteed to the mean value of the process; the covariance (belief) among agents possibly diverging \cite{kamgarpour2008convergence}. This issue is sometimes addressed in practice by covariance intersection methods \cite{julier1997non,chen2002estimation}. When dealing with non-Gaussian and nonlinear processes, the Bayes filter is known to be the optimal recursive estimation procedure \cite{chen2003bayesian}. However, guaranteeing the convergence of distributed versions of the Bayes filter over general networks poses severe challenges. In this regard, recent approaches \cite{bandyopadhyay2014distributed,bandyopadhyay2018distributed} have proposed the use of logarithmic opinion pooling mechanisms to overcome this issue.

Another further concern with both distributed Kalman and Bayes filters is their use of communication resources. Some methods \cite{olfati2009kalman,bandyopadhyay2014distributed}, require agents to exchange information multiple times for each filtering iteration, always maintaining consensus between observations. Other methods drop this requirement, allowing for a single round of communication per each observation \cite{kamal2013information,das2016consensus}. Still, the desire to further reduce the communication use has led to event-triggered approaches to distributed estimation \cite{nowzari2019event}. Attention has centered on variations of the Kalman consensus filter \cite{olfati2009kalman}, such as the diffusion Kalman filter \cite{cattivelli2010diffusion}, which can support censoring mechanisms. Triggering conditions provided by a threshold on the innovations \cite{ouimet2018cooperative} or the estimation error \cite{liu2015event,meng2014optimality} being studied. However, these methods suffer from the same estimate consistency issues of previous Kalman consensus algorithms \cite{kamgarpour2008convergence}.

The approach of this work was motivated by recent advances in distributed Bayesian filtering \cite{bandyopadhyay2018distributed}, which exploit logarithmic opinion pools (LogOP) to produce efficient distributed algorithms. We augment these approaches by introducing a Value of Information (VoI) \cite{mu2014efficient} censoring procedure. This mechanism measures the difference between a node's local estimate and its neighborhood average, transmitting when the KL-divergence between these estimates exceeds a certain threshold. Furthermore, in order to obtain a simple and computationally practical algorithm, we specialize the Bayesian VoI filter to linear-Gaussian dynamics, obtaining a censored and distributed variant of the Kalman information filter. In contrast to approaches based on the consensus Kalman algorithm \cite{meng2014optimality,liu2015event,ouimet2018cooperative}, the proposed method provides estimates that are guaranteed to be consistent, preventing the divergence of estimate belief across agents \cite{kamgarpour2008convergence}. Importantly, we show that consistent estimates are necessary in order to produce efficient censoring decisions. To validate the methodology, we conduct simulations of the proposed filtering approach in a target tracking problem over a complex time-varying communication network. We further validate our approach by performing an experimental study over real ad-hoc wireless communications. Results from both simulations and experiments show the validity of using the proposed method to do distributed estimation while only using a fraction of the communication resources.

\section{Problem Formulation}

In this work, we study the problem of distributed estimation. To this end, consider a set of $N$ agents acquiring observations from a random process. Each of the agents has access to its own probability density function (pdf) $p_i$ representing the belief of the $i$-th agent in its estimate of the underlying process. The objective of the agents is to collectively reach a consensus of the following form
\begin{align}\label{eq:minKL}
p^\star \triangleq \argmin_p & \sum_{i=1}^{N} D_{\text{KL}} \bigl( p \| p^i\bigr),
\end{align}
where $D_{\text{KL}}(p \| q)= \int p(x) \log \left(\frac{p(x)}{q(x)}\right) dx$ is the Kullback-Leibler (KL) divergence. More specifically, the objective is to design an iterative consensus algorithm such that $\lim_{k \to \infty} p^i_k = p^\star$ for all agents, where $p^i_k$ denotes the pdf of the $i$-th node at the $k$-th time instance. We consider distributions $p(x)$ generated by a dynamical system. To this end, let $x_k$ correspond to the state of a process at time $k$ evolving according to the following dynamics
\begin{align}\label{eq:system_dynamics}
x_{k+1}=f_k \left(x_k, w_k \right),
\end{align}
where $w_k$ is an independent and identically distributed (i.i.d.) process noise and $f_k$ is a possibly nonlinear and time-varying function governing the state transitions of the system. Each of the $N$ agents may acquire measurements from the process, where the observation of the $i$-th sensor at time $k$ is 
\begin{align}\label{eq:system_measurement}
z^{i}_{k}=h^{i}_{k} \left(x_k, v^{i}_{k} \right), \quad i=1\ldots,N
\end{align}
where $h^{i}_{k}$ is, again, a possibly nonlinear and time-varying function of the state $x_k$ and $v^{i}_{k}$ is i.i.d. measurement noise. To perform the estimation task, the sensors communicate over a communication network given by the graph  $\mathcal{G}=(\mathcal{N},\mathcal{E})$, where $\mathcal{N}$ is the set of $N$ nodes in the network and $\mathcal{E} \subseteq \mathcal{N} \times \mathcal{N}$ is the set of communication links, such that if node $i$ is capable of communicating with node $j$, we have $(i,j) \in \mathcal{E}$. Moreover, we define the neighborhood of node $i$ as the set $\mathcal{N}_i=\{j | (i,j)\in \mathcal{E}\}$. 

\subsection{Bayesian Filtering and Opinion Pooling}

In order to obtain an estimate of the state $x_k$ in an iterative manner we resort to Bayesian filtering \cite{chen2003bayesian}. By making use of the Chapman-Kolmogorov equations, we can write the filter prediction step at the $i$-th node as 
\begin{align}\label{eq:prior}
p^i_k(x_k)=\int p_k^i(x_k | x_{k-1}) p_{k-1}^i(x_{k-1}) dx_{k-1},
\end{align}
where the distribution $p_k^i(x_k | x_{k-1})$ is governed by the system dynamics \eqref{eq:system_dynamics}. The new measurement $z_k^i$ is then used to compute the update step of the filter, given by Bayes' rule
\begin{align}\label{eq:posterior}
p^i_k(x_k | z_k^i)=\frac{p_k^i(z_k^i | x_k) p_k^i(x_k)}{\int p_k^i(z_k^i | x_k) p_k^i(x_k) dx_k}.
\end{align}
The iterative application of  \eqref{eq:prior}$-$\eqref{eq:posterior} form the Bayes' filtering algorithm. These expressions use only the local information available at each agent (the measurement $z_k^i$). A naive way of including the information of other agents in the network would be for agents to share all their measurements, leading to a centralized estimate $p^i_k(x_k | z_k^1,\ldots,z_k^N)$. However, this approach is communication intensive, as measurements must be propagated across the network. Intuition for a better approach can be gained from the optimal solution to the KL consensus problem \eqref{eq:minKL}, given by the following proposition.
\begin{proposition}[Logarithmic Opinion Pooling]\label{prop:logop}
The optimal solution of the KL consensus problem \eqref{eq:minKL}, namely $p^\star=\argmin_{p} \sum_{i=1}^{N} D_{\text{KL}}(p \| p_i)$ is given by
\begin{align}\label{eq:logop}
p^\star (x)=\frac{\displaystyle\prod\limits_{i \in \mathcal{N}_i} \biggl[p^i(x)\biggr]^{1/|\mathcal{N}_i|}}
{\displaystyle\int \prod\limits_{i \in \mathcal{N}_i} \biggl[p^i(x)\biggr]^{1/|\mathcal{N}_i|} dx}
\end{align}
\end{proposition}
\begin{proof}
See \cite[Proposition 1]{bandyopadhyay2014distributed}.	
\end{proof}
In statistics this is known as the Logarithmic Opinion Pool (LogOP) operator \cite{genest1986combining}. The LogOP has several properties that make it attractive to being used iteratively. Importantly, the repeated application of this operator leads to consensus across the network. Motivated by this, we can conduct local Bayesian steps \eqref{eq:prior}$-$\eqref{eq:posterior} and perform LogOP consensus, where instead of aggregating over all nodes, agents communicate only with their set of neighbors $\mathcal{N}_i$. This can be guaranteed to converge to the optimal solution, as stated next
\begin{proposition}\label{prop:convergence}
Under mild conditions, the iterative application of the local Bayes filtering steps \eqref{eq:prior}$-$\eqref{eq:posterior} plus logarithmic opinion pooling \eqref{eq:logop}, with $\mathcal{N}=\mathcal{N}_i$ converges to the optimal solution $\lim_{k \to \infty} p^i_k = p^\star$ for all $i=1,\ldots,N$ agents.
\end{proposition}
\begin{proof}
See \cite[Theorem 5]{bandyopadhyay2018distributed}. 
\end{proof}
\section{Value of Information (VoI)}

The usage of the LogOP operator allows to reduce network communication to a single broadcast transmission by each node per time step. Nonetheless, this can still result in a large amount of network usage. In order to further reduce the number of communication events we introduce the notion of Value of Information (VoI) \cite{mu2014efficient}. This notion is based on quantifying the information acquired at a given time instance and have only nodes with high VoI transmit. The VoI is measured by the KL divergence between the neighborhood fused and local distributions. This results in VoI transmit decisions when
\begin{align}\label{eq:voi}
D_{\text{KL}} \left( p^i_k(x_k) ~\|~ \tilde{p}^i_k(x_k) \right) \geq \gamma
\end{align}
where $\gamma$ denotes the censoring level and $\tilde{p}^i_k(x_k)$ corresponds to the local distribution obtained without performing the LogOP aggregation. For VoI under this level, the agent censors itself. The resulting algorithm is described in Algorithm \ref{alg:bayes}.

\begin{algorithm}[t]\caption{VoI Bayes Filtering}\label{alg:bayes} 
\begin{algorithmic}[1]
 \renewcommand{\algorithmicrequire}{\textbf{Input:}}
 \renewcommand{\algorithmicensure}{\textbf{Output:}}
 \For {$k=0,1\ldots$}
 \State Compute update using the measurement $z_k^i$
\begin{align*}
p^i_k(x_k | z_k^i)=\frac{p_k^i(z_k^i | x_k) p_k^i(x_k)}{\int p_k^i(z_k^i | x_k) p_k^i(x_k) dx_k}
\end{align*}
\State Transmit $p^i_k(x_k)$ if
\begin{align*}
D_{\text{KL}} \left( p^i_k(x_k) ~\|~ \tilde{p}^i_k(x_k) \right) \geq \gamma
\end{align*}
\State Aggregate neighboring estimates via LogOP
\begin{align*}
\bar{p}^i_k(x) =\frac{\displaystyle\prod\limits_{i \in \mathcal{N}_i} \bigl[p_k^i(x)\bigr]^{1/|\mathcal{N}_i|}}
{\displaystyle\int \prod\limits_{i \in \mathcal{N}_i} \bigl[p_k^i(x)\bigr]^{1/|\mathcal{N}_i|} dx}
\end{align*}
\State Compute local prediction step
\begin{align*}
\tilde{p}^i_k(x_k)=\int p_k^i(x_k | x_{k-1}) p_{k-1}^i(x_{k-1}) dx_{k-1}
\end{align*}  
\State Compute aggregated prediction step
\begin{align*}
p^i_k(x_k)=\int \bar{p}_k^i(x_k | x_{k-1}) \bar{p}_{k-1}^i(x_{k-1}) dx_{k-1}
\end{align*}  
\EndFor
\end{algorithmic}
\end{algorithm}

\subsection{VoI Kalman Filter}

While the Bayesian formulation provides intuition and guidance as to algorithmic design, implementing the Bayes' filter steps \eqref{eq:prior}$-$\eqref{eq:posterior} poses several computational challenges. In practice, particle filters \cite{arulampalam2002tutorial} or similar techniques are used to represent the distributions. Another possible approach is to consider simpler linear-Gaussian models. In this case, the Bayesian VoI takes a form reminiscent of a censored Kalman information filter \cite{battistelli2018distributed}. More specifically, consider the following dynamic process evolving over time $k$.
\begin{align}
x_{k+1}=A_k x_k + w_k,
\end{align}
where the noise is given by $w_k \sim \mathcal{N}(0,Q_k)$. Measurements are obtained by a set of $N$ sensors with the $i$-th sensor at time $k$ obtaining the following observation
\begin{align}
z^i_{k}=H^i_{k} x_k + v^i_{k},
\end{align}
with $v_{k}^i \sim \mathcal{N}(0,R_{k}^{i})$. Then, the Bayes steps can be particularized to obtain the filtering procedure described in Algorithm \ref{alg:kalman}. The linear-Gaussian case provides several computational advantages. Among them, the KL divergence can be computed in closed form as stated in the following proposition.

\begin{algorithm}[t]\caption{VoI Kalman Filtering}
\label{alg:kalman} 
\begin{algorithmic}[1]
 \renewcommand{\algorithmicrequire}{\textbf{Input:}}
 \renewcommand{\algorithmicensure}{\textbf{Output:}}
   \State Use information vector and matrix forms
	\begin{align*}
		y^i_{k|k}=(P^i_{k|k})^{-1} x_{k|k} \quad\quad Y^i_{k|k}=(P^i_{k|k})^{-1}
	\end{align*}
 \For {$k=0,1\ldots$}
  \State Compute update using the measurement $z_k^i$
\begin{align*}
	 i^i_k &= (H_k^i)^T (R_k^i)^{-1} z_k^i  & I^i_k &= (H_k^i)^T (R_k^i)^{-1} H_k^i \\
    y^i_{k|k} &= y^i_{k|k-1} + i^i_k        & Y^{i}_{k|k} &= Y^{i}_{k|k-1} + I^i_k
\end{align*}
  \State Transmit $(y^i_{k|k}, Y^{i}_{k|k})$ if
  \begin{align*}
	D_{\text{KL}}\bigl(\mathcal{N}(x^{i}_{k|k}, P^{i}_{k|k}) ~\|~ \mathcal{N}(\tilde{x}^{i}_{k|k-1}, \tilde{P}^{i}_{k|k-1})\bigr) \geq \gamma
  \end{align*}
  \State Aggregate neighboring estimates via LogOP
\begin{align*}
\bar{y}^{i}_{k|k} &= \frac{1}{| \mathcal{N}_i |} \biggl( y^{i}_{k|k} + \sum_{j \in \mathcal{N}_i} y^{j}_{k|k} \biggr) \\
\bar{Y}^{i}_{k|k} &= \frac{1}{| \mathcal{N}_i |} \biggl( Y^{i}_{k|k} + \sum_{j \in \mathcal{N}_i} Y^{j}_{k|k}  \biggr)
\end{align*}	
  \State Compute local prediction step
\begin{align*}
\tilde{M}^i_k &= (A_k^{-1})^T Y^i_{k|k} A_k^{-1}\\
\tilde{y}^i_{k+1|k} & = (I+M^i_k Q_k)^{-1} (A_k^{-1})^T y^i_{k|k} \\ 
\tilde{Y}^{i}_{k+1|k} &= (I+M^i_k Q_k)^{-1} M^i_k
\end{align*}	
  \State Computed aggregated prediction step
\begin{align*}
M^i_k &= (A_k^{-1})^T \bar{Y}^i_{k|k} A_k^{-1}\\
y^i_{k+1|k} & = (I+M^i_k Q_k)^{-1} (A_k^{-1})^T \bar{y}^i_{k|k} \\ 
Y^{i}_{k+1|k} &= (I+M^i_k Q_k)^{-1} M^i_k
\end{align*}	

\EndFor
\end{algorithmic}
\end{algorithm}

\begin{proposition}\label{prop:klgaussian}
The KL divergence between the estimates $\mathcal{N}(x^{i}_{k|k}, P^{i}_{k|k})$ and $\mathcal{N}(\tilde{x}^{i}_{k|k-1}, \tilde{P}^{i}_{k|k-1})$ is given by
\begin{align}
&D_{\text{KL}}\left(\mathcal{N}(x^{i}_{k|k}, P^{i}_{k|k}) \| \mathcal{N}(\tilde{x}^{i}_{k|k-1}, \tilde{P}^{i}_{k|k-1})\right) \nonumber\\ 
& = \frac{1}{2} \biggl( \bigl( \tilde{x}^i_{k|k-1}  - x^i_{k|k} \bigr)^T \bigl(\tilde{P}^{i}_{k|k-1}\bigr)^{-1}\bigl( \tilde{x}^i_{k|k-1}  - x^i_{k|k} \bigr) \nonumber\\
& + \tr \biggl( \bigl(\tilde{P}^{i}_{k|k-1}\bigr)^{-1} P^{i}_{k|k}\biggr) - n + \log\biggl( \frac{\det \tilde{P}^{i}_{k|k-1} }{ \det P^{i}_{k|k} } \biggr) \biggr)
\end{align}
\end{proposition}
\begin{proof}
See Appendix \ref{app:klgaussian}.	
\end{proof}

Another important property is the fact that the estimate pairs $(x_{k|k}, P_{k|k})$ are consistent for all iterates, as stated next.
\begin{proposition}\label{prop:consistency}
The iterates $(x^i_{k|k}, P^i_{k|k})$ generated by Algorithm \ref{alg:kalman} are consistent. That is, for all $k$ and $i$, we have
\begin{align}
\mathbb{E}\left[\left(x-x^i_{k|k}\right) \left(x-x^i_{k|k}\right)^T \right] \leq P^i_{k|k}.
\end{align}
\end{proposition}
\begin{proof}
See Appendix \ref{app:consistency}.	
\end{proof}
Thus, the proposed estimation mechanism avoids the possible divergence of belief across the network. This is an issue that plagues consensus-based Kalman filters \cite{olfati2009kalman,olfati2007distributed,kamgarpour2008convergence,cattivelli2010diffusion}, which focus only in guaranteeing consensus of the estimate and not the belief of the agents. Furthermore, as numerically explored in the next section. The consistency of estimates plays an important role in enabling intelligent censoring decisions.

\begin{figure*}[t!]	
	\centering
	\includegraphics[scale=1]{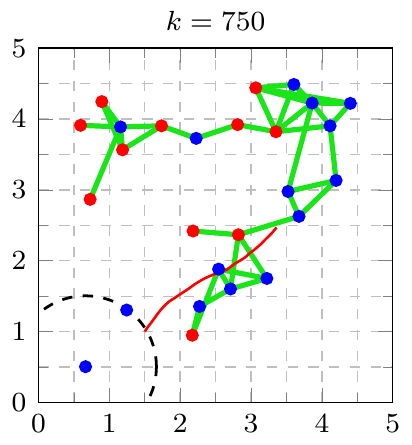} 					
	\includegraphics[scale=1]{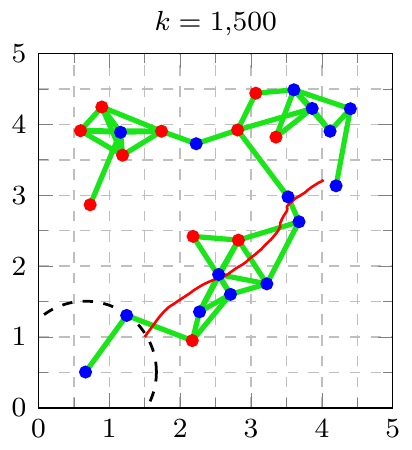} 					
	\includegraphics[scale=1]{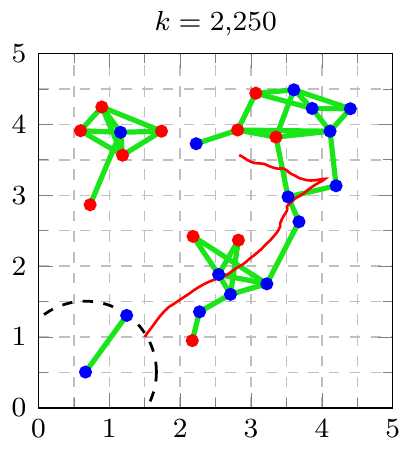} 					
	\includegraphics[scale=1]{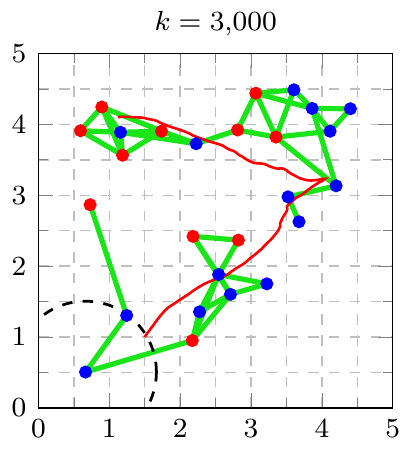} 					
	\caption{Node positions and resulting time-varying network connectivity. Four snapshots are shown, corresponding to time instances $k=750$, $k=1{,}500$, $k=2{,}250$ and $k=3{,}000$. Distances in the axis are in kilometers. Blue dots correspond to ranging (TOA) nodes and red dots correspond to bearing (DOA) nodes. The sensing radius of an agent is 1 km, shown by the dashed circle centered on the lower left node. The trajectory of the target is shown in red. The initial state of the target is $x_0=[1500, 8, 1000, 12]^T$ and at time $k=1{,}500$ it changes trajectory by setting its velocities to $(\dot{x},\dot{y})=(-8,0.1)$.}
	\label{fig:dynamic_map_connectivity}
\end{figure*}


\section{Numerical Results}
\label{sec:numericals}

This section numerically investigates the properties of the proposed distributed estimation algorithm. We evaluate the filtering algorithm in a target tracking problem \cite{bandyopadhyay2018distributed,bar2004estimation}. Assume that the linear-Gaussian target dynamics evolve as 
\begin{equation}
x_{k+1}=A x_k + w_k,
\end{equation}
where the state of the system at the $k$-th time instance $x_k=[x,\dot{x},y,\dot{y}]^T$ is composed of is its position $(x,y)$ along the x- and y-axis, together with its corresponding velocities $(\dot{x},\dot{y})$. The dynamics follow a nearly-constant velocity model with
\begin{align}
A=
\begin{bmatrix}
1 & \Delta & 0 & 0\\
0 & 1 & 0 & 0\\
0 & 0 & 1 & \Delta \\
0 & 0 & 0 & 1
\end{bmatrix},
\end{align}
where $\Delta$ corresponds to the sampling interval. Furthermore, the process noise is Gaussian according to $w_k \sim \mathcal{N}(0,Q)$, where the covariance matrix $Q$ is given by \cite{bar2004estimation}
\begin{align}
Q=
\begin{bmatrix}
\frac{\Delta^3}{3} & \frac{\Delta^2}{2} & 0 & 0\\
\frac{\Delta^2}{2} & \Delta & 0 & 0\\
0 & 0 & \frac{\Delta^3}{3} & \frac{\Delta^2}{2}\\
0 & 0 & \frac{\Delta^2}{2} & \Delta
\end{bmatrix}.
\end{align}
A network of agents is tasked with tracking this target. We consider a random heterogenous deployment of both Time-Of-Arrival (TOA) and Direction-Of-Arrival (DOA) nodes. These nodes only obtain measurements when the target is in their limited sensing range. When the target is inside this range, the measurement acquired is
\begin{align}\label{eq:measurement_model_num}
z^{i}_{k}=
\begin{cases}
\sqrt{(x_k - x^i)^2 + (y_k - y^i)^2} +v_{k}^{i,r}  &\text{if TOA}\\
\arctan\bigl(\frac{x_k-x^i}{y_k-y^i} \bigr)+v_{k}^{i,\theta} &\text{if DOA}
\end{cases}
\end{align}
where $(x^i,y^i)$ denotes the position of the $i$-th agent in the network. The measurement noise is distributed according to $v_{k}^{i,r} \sim \mathcal{N}(0,\sigma_r)$ with $\sigma_r=1.5$ m, and $v_{k}^{i,\theta} \sim \mathcal{N}(0,\sigma_\theta)$ with $\sigma_\theta= 2^{\circ}$. The usual extended Kalman filter approach is used to address the nonlinearity of the measurement model.

In order to track the target, the nodes conduct distributed estimation using the VoI filter (Algorithm \ref{alg:kalman}). As part of this process, they share information with each other using a communication network. In order to simulate a realistic communication environment in which VoI filtering can prove useful, we resort to underwater acoustic communication channels. The underwater communication channel suffers from severe fading and is low bandwidth, rendering traditional (non-censoring) distributed estimation approaches difficult to successfully operate in practice. To obtain this communication network we resort to a statistical characterization of the underwater channel \cite{qarabaqi2013statistical}, including underwater path-loss, multi-path propagation via raytracing, deviations due to underwater current drift and deviations due to surface waves. The received signal strength at each link is used to detect for each transmission a successful or unsuccessful message delivery. This results in a dynamic network topology as shown in Fig. \ref{fig:dynamic_map_connectivity}, wherein the network is not always connected, nodes become isolated at times, yet the network becomes periodically strongly connected.

\begin{figure}[t]
	\centering
	\includegraphics[scale=1]{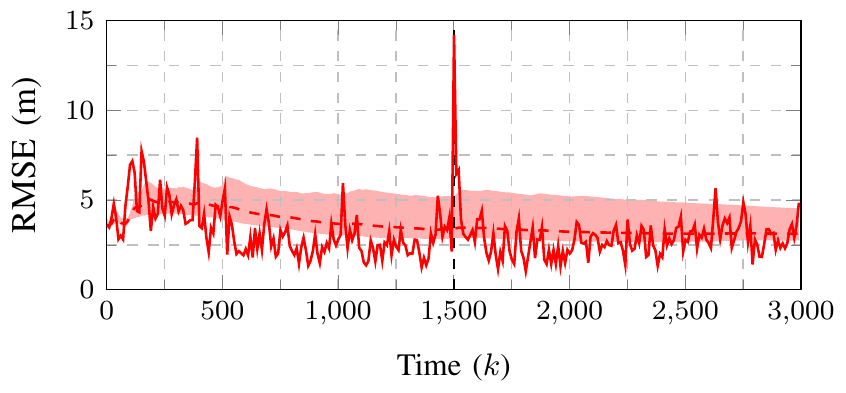} 
	\vspace*{-8mm}
	\caption{Estimation error (RMSE) averaged across the network nodes. The error at the $k$-th time instance is shown by the solid line and the dashed line corresponds to the running average of the estimation error. The shaded area covers the best and worst individual sensor estimation error across the network. The VoI filter operates at a censoring level of $\gamma=0.4$. A spike in estimation error occurs immediately after the maneuver event at $k=1{,}500$, from which the system recovers rapidly.}
	\label{fig:dynamic_mse}
\end{figure}

We conduct simulations in the environment shown in Fig. \ref{fig:dynamic_map_connectivity}, wherein the task of the nodes is to track the target. We first evaluate the estimation error of the task, plotted in Fig. \ref{fig:dynamic_mse}. Overall, the RMSE across the network is shown to be asymptotically stable, hence the VoI filter is capable of tracking the target. Due to the nature of the network and the censoring process, nodes converge to a region surrounding the average asymptotic RMSE. An important event in this simulation is the change of direction that the target conducts at $k=1{,}500$. Since the agents have no information about this maneuver and use a linear model to compute their prediction step, the RMSE immediately spikes. Nonetheless, the measurements obtained immediately after the target changes course have high VoI, and result in an increase of transmissions that results in a quick correction of the estimates.

\begin{figure}[t]
	\centering
	\includegraphics[scale=1]{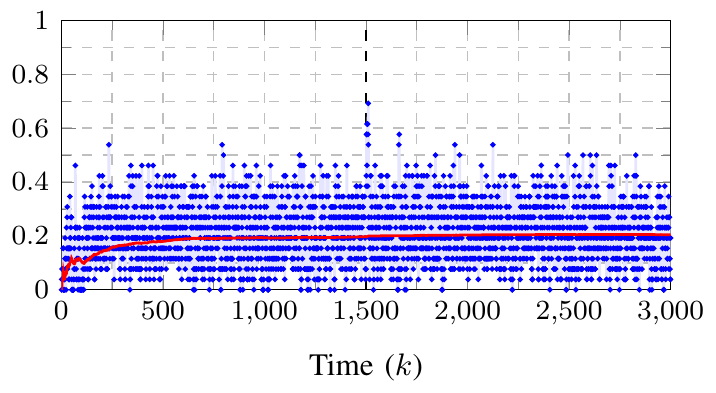} \hspace{-4ex}			
	\includegraphics[scale=1]{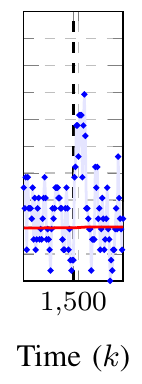} 	
	\vspace*{-4mm}
	\caption{Medium access averaged across the network nodes, i.e., fraction of agents deciding to communicate at a given time instance. Each of the dots corresponds to a time instance, while the average is shown in red. On the right, we focus on the maneuver event at $k=1{,}500$. After the maneuver, communication increases briefly to quickly correct the estimates.}
	\label{fig:dynamic_tx}
\end{figure}

We can also observe how the system makes use of the communication resources to react to unexpected events. In Figure \ref{fig:dynamic_tx} we plot the medium access decisions across the network. On average, around $20\%$ of the nodes transmit at any given instant. However, there is high variability across time (e.g., many time instances with complete radio silence by all the nodes). An important observation of the simulation is that the change of trajectory at $k=1{,}500$, besides creating a spike in RMSE as shown in Fig. \ref{fig:dynamic_mse}, also causes a short-lived increase in communication in order to re-stabilize the estimation error. As the change of direction produces measurements with higher VoI, nodes become more active in order to propagate this new information across the network. During this simulation, the time instance with the highest level of medium access, around $70\%$, occurs after this event.

Finally, in order to fully evaluate the performance of the VoI filter, we plot the trade-off between RMSE and medium access, shown in Fig. \ref{fig:dynamic_gamma}. This plot is created by sweeping over a range of possible censoring values $\gamma$. The value $\gamma=0$ corresponds to a system without censoring (all nodes are always transmitting), and as the value of $\gamma$ increases, nodes begin to censor and the RMSE progressively increases, resulting in the operating region shown in Fig. \ref{fig:dynamic_gamma}. This results in a boundary of possible operating points for the VoI algorithm. For example, the value $\gamma=0.4$, used in the previous simulations (Fig. \ref{fig:dynamic_mse} and \ref{fig:dynamic_tx}) uses only $19\%$ of the communication resources with performance similar to the non-censoring case. Further reductions in communication usage become progressively more expensive in resulting RMSE level as the system approaches the knee of its operating curve. Ultimately, the choice of $\gamma$ depends on the application and the desired trade-off between estimation error and communication usage. 

Furthermore, observe that both the VoI censoring decision and the consistency of the estimates play a crucial role in the required communication resources. In Fig. \ref{fig:dynamic_gamma} we have also characterized the result of using the proposed VoI filtering algorithm without covariance information in the censoring decision ($P_{k|k}^{i}=I$ and $\tilde{P}^{i}_{k|k-1}=I$ in Proposition \ref{prop:klgaussian}). This corresponds to a censoring decision dictated simply by the euclidean norm $\| x^{i}_{k|k} - \tilde{x}^{i}_{k|k-1} \|$. The result is an increase of the average error. More importantly, the variance across the network nodes spreads drastically, highlighting the importance of including the estimate belief in the censoring decision. However, this algorithm is still based on the LogOP operation, yielding consistent estimates. This requires agents to communicate covariance matrices with their neighbors. In order to see the effect of a lack of consistency, we consider now the diffusion Kalman filter \cite{cattivelli2010diffusion}. This standard filtering methodology does not communicate covariance matrices (and therefore produces inconsistent estimates), however it can support censoring decisions (which standard consesus Kalman filters do not \cite{olfati2009kalman}). In this case, the lack of consistent estimates results in an even larger increase of the estimation error, with a drastically reduced operating range of the filter (e.g., regions under $10\%$ of communication usage are unachievable).

\begin{figure}[t]
	\centering
	\includegraphics[scale=1]{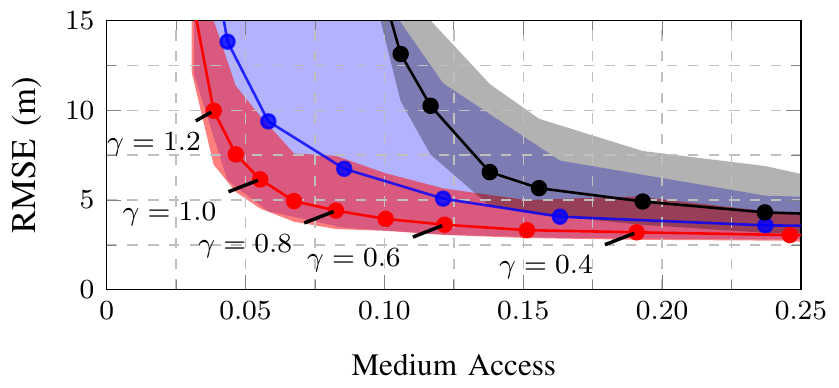} 			
\vspace*{-4mm}
\caption{Operating region of the VoI filter shown by the red line. Asymptotic MSE and medium access are obtained for different values of the censoring level $\gamma$. The shaded area correspond to the range of RMSE obtained by the best and worst individual sensor in the network. The blue curve corresponds to a VoI filter censoring without covariance information, i.e., $P_{k|k}^{i}=I$. The black curve corresponds to performing VoI filtering in a diffusion Kalman filter, which does not produce consistent estimates.}
	\label{fig:dynamic_gamma}
\end{figure}

\section{Experiments}

To further validate the feasibility of the proposed distributed filtering approach we conducted experimental evaluations. These tests were conducted using a set of Qualcomm Flight Pro platforms, an integrated computing board powered by a Qualcomm Snapdragon 820 processor. These boards integrate a Qualcomm QCA6174A communication chip, providing IEEE 802.11 networking capabilities. Communications were configured to operate in IBSS (ad-hoc) mode. Thus, the nodes provide their own communication network, without the need to resort to any previously available wireless communication infrastructure. The filtering algorithm is implemented in ROS and fully runs onboard the nodes in a distributed manner.

We deployed 5 nodes as shown in Fig.~\ref{fig:env}. The placement of the nodes provides a complex communication environment, providing line-of-sight links as well as the need to relay messages across the network. The agents are tasked with estimating a virtual target traversing the environment and following the dynamics introduced in Section \ref{sec:numericals}. The agents are synchronized and share the same random seed, hence they see the same trajectory. To this end Nodes 1, 3 obtain ranging measurements and Node 4 obtains bearing measurements, while Nodes 2 and 5 do not have any sensing capabilities. The measurement model follows \eqref{eq:measurement_model_num}. To track the target, the agents use the VoI filter with censoring level $\gamma=0.15$ (Algorithm \ref{alg:kalman}) over a real wireless network.

The resulting estimation error of this task, measured by the RMSE, is shown in Fig. \ref{fig:exp_mse_nodes}. As expected, the estimation error across the network (shown by the dashed line) is asymptotically stable, thus the agents successfully track the target. Further, we also plot individual instantaneous errors of the agents with sensing capabilities. Nodes 3 and 4 have similar errors, which can be expected as they occupy similar (mirrored) positions in the network topology. In contrast, Node 1 exhibits instances of higher RMSE, which occur due to its isolation in the network, being only able to communicate with Node 2. In any case, this limited communication with a single neighbor is sufficient, as Node 2 achieves asymptotic stability in its estimation error.

Similar behavior can be observed by looking at the bandwidth usage of the algorithm, shown in Fig. \ref{fig:exp_tx_all}. This corresponds to the actual data traffic generated by the nodes using the VoI filter in order to perform the estimation task. Nodes 3 and 4 have similar network usage, while Node 1, being somewhat isolated in the network, has less network usage. Further, observe that nodes acquiring measurements (source nodes) contribute to the majority of the traffic in the network, rather than the nodes that can only act as relays (black solid vs dashed line). Importantly, the estimation tasks requires very little communication resources ($1.67$ kbps).

This is better evaluated by the trade-off between estimation error and data rate. In a manner similar to Fig. \ref{fig:dynamic_gamma} in the previous section, we have now experimentally characterized the operating region of the VoI filter, shown in Figure \ref{fig:exp_gamma}. As previously seen in simulation, error is traded to obtain a reduction in communication usage. Nonetheless, large reductions in data rate can be obtained without significant impact on the RMSE. This is the case of the previous experiment with censoring level $\gamma=0.15$. That experiment utilized $1.67$ kbps obtaining similar RMSE to an experiment without censoring ($\gamma=0$). However, the latter required $18.03$ kbps to perform the same estimation task.

\begin{figure}[t]
	\centering
	\includegraphics[scale=1]{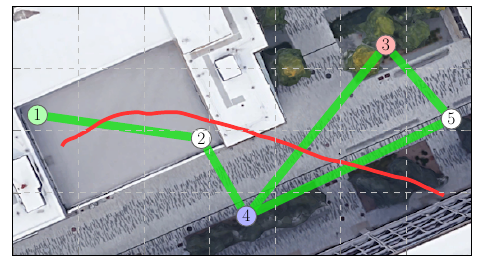} \hspace{-2ex}			
	\includegraphics[scale=1]{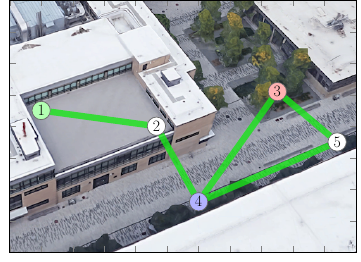} 	
	\vspace*{-4mm}
	\caption{Deployment location of the nodes in the experiment. The area corresponds to building 31 and its surroundings at the MIT campus. Nodes 1 and 2 are located atop of a building. The grid shown on the left corresponds to a spacing of 10 m. Line-of-sight connections between the nodes are shown in green. The trajectory of the target is shown in red.}
	\label{fig:env}
\end{figure}

\begin{figure}[t]
	\centering
	\includegraphics[scale=0.975]{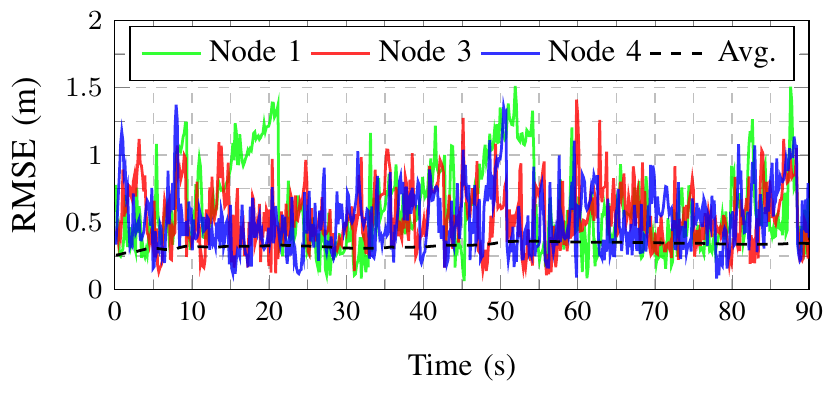} 		
	\vspace*{-4mm}
\caption{Estimation error resulting from a the experiment. Instantaneous RMSE is shown for the source Nodes $1$, $3$ and $4$. The running average over all the nodes is shown by the dashed line. The system is asymptotically stable across the network.}
	\label{fig:exp_mse_nodes}
%
	\centering
	\vspace*{2mm}	
	\includegraphics[scale=0.975]{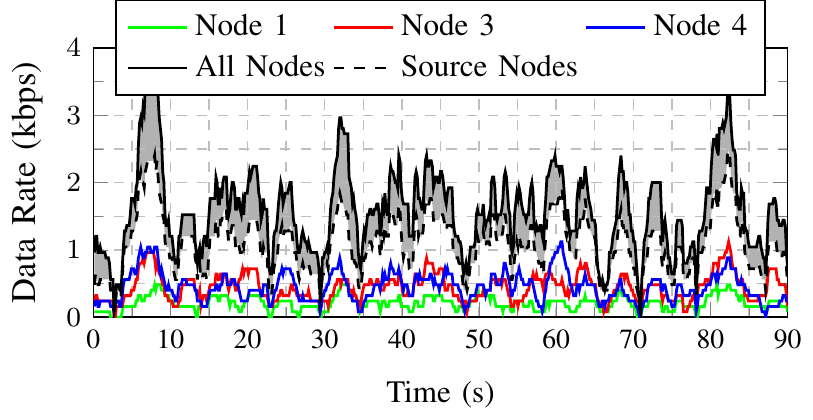} 			
	\vspace*{-4mm}
	\caption{Bandwidth usage of the estimation task during the experiment.}
	\label{fig:exp_tx_all}
%
	\centering
	\includegraphics[scale=0.975]{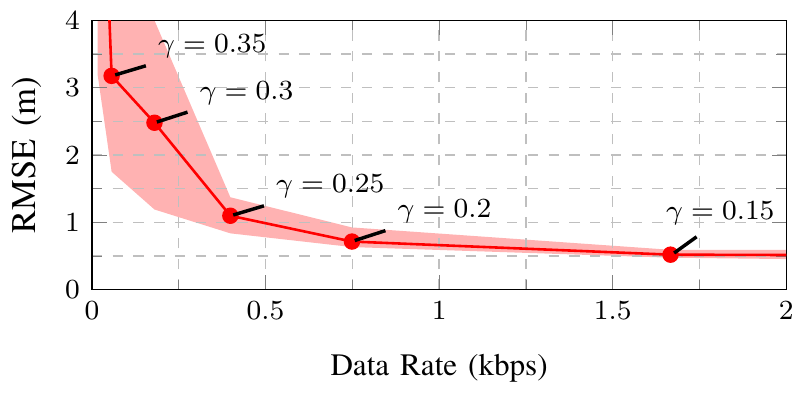} 
	\vspace*{-4mm}
	\caption{Operating region of the VoI filter. RMSE vs total data rate consumed in the network. A non-censoring ($\gamma=0$) approach utilized $18.03$ kbps.}
	\label{fig:exp_gamma}
\end{figure}

\section{Conclusions}

In this paper, we have proposed a distribution estimation algorithm with censoring decisions based on Value of Information (VoI). For linear-Gaussian dynamics, this approach has been shown to take a form similar to a censored distributed Information filter. Importantly, the filtering process produces consistent estimates across all nodes in the network. The VoI censoring decision plus consistency in the estimates plays a crucial role in allowing for large reductions in communication usage. This behavior has been verified in target tracking simulations over complex time-varying networks. Furthermore, these results have also been experimentally validated over wireless networks. 

\bibliographystyle{ieeetr}
\bibliography{bib}


\appendix

\section{Appendix}

\subsection{Proof of Proposition \ref{prop:klgaussian}}\label{app:klgaussian}
\begin{proof}
First, recall that the density for a Gaussian distribution with mean $\hat{x}$ and covariance matrix $\hat{P}$ is given by
\vspace{-1ex}
\begin{align*}
p(x)=\frac{1}{\sqrt{(2 \pi)^n \det \hat{P}}} 
\exp \left(-\frac{1}{2}
\bigl(x - \hat{x} \bigr)^T \hat{P}^{-1} \bigl(x - \hat{x} \bigr),
 \right)
\end{align*}
\vspace{-1ex}
and the Kullback-Leibler divergence can be rewritten as
\begin{align*}
D_{\text{KL}}(p \| q)&= \int p(x) \log \left(\frac{p(x)}{q(x)}\right) dx =\mathbb{E}_p \left[\log p - \log q\right]
\end{align*}
Particularizing the KL divergence for the case of two Gaussian distributed variables $(x_1,P_1)$ and $(x_2,P_2)$, we have 

\begin{align}
D_{12}& \triangleq D_{\text{KL}}\left(\mathcal{N}(x_1, P_1) ~\|~ \mathcal{N}(x_2, P_2)\right)= \nonumber \\
& \frac{1}{2}\mathbb{E}_1 \biggl[ - \log\det P_1 -\left(x-x_1\right)^T P_1^{-1} \left(x-x_1\right)  \nonumber\\
& \quad \quad + \log\det P_2 + \left(x-x_2\right)^T P_2^{-1} \left(x-x_2\right)\biggr].
\end{align}
By factoring terms factor and using the fact that we can write $\left(x-x_1\right)^T P_1^{-1} \left(x-x_1\right)=\tr \left(\left(x-x_1\right)^T P_1^{-1} \left(x-x_1\right)\right)$, and further using the cyclic property of the trace operator, $\tr (ABC)=\tr (BCA)=\tr (CAB)$, we can rewrite
\begin{align}
D_{12} = \frac{1}{2}
\mathbb{E}_1 \biggl[
& - \tr\left(P_1^{-1}\left(x-x_1\right)\left(x-x_1\right)^T \right) \\
& + \tr\left(P_2^{-1}\left(x-x_2\right) \left(x-x_2\right)^T\right) 
+ \log \frac{\det P_2}{\det P_1} \biggr] \nonumber
\end{align}
Then by taking the expectation, expanding the first two terms and using the fact that $\mathbb{E}\left[(x-y)^TA(x-y)\right]=(\mu_x-y)^T A (\mu_x-y)+\tr(A\Sigma_x)$ we write
\begin{align}
D_{12} = \frac{1}{2} \biggl[
&  \left(x_2-x_1\right)^T P_2^{-1} \left(x_2-x_1\right) \nonumber\\
&  \quad\quad+ \tr \left(P_2^{-1} P_1\right) - n + \log \frac{\det P_2}{\det P_1} \biggr] 
\end{align}
Then substituting $(x_1,P_1)=(x^{i}_{k|k}, P^{i}_{k|k})$ and $(x_2,P_2)=(\tilde{x}^{i}_{k|k-1}, \tilde{P}^{i}_{k|k-1})$ leads to the desired expression.
\end{proof}

\subsection{Proof of Proposition \ref{prop:consistency}}\label{app:consistency}
\begin{proof}
In order to verify consistency, the first step is to guarantee that the LogOP merging operation is consistent. Remember the definitions of the information forms used in Algorithm \ref{alg:kalman}. Namely, $y^i_{k|k}=(P^i_{k|k})^{-1} x^{i}_{k|k}$ and $Y^i_{k|k}=(P^i_{k|k})^{-1}$. This allows us to rewrite the LogOP update as
\begin{align}
\bar{y}^{i}_{k|k} &= \frac{1}{| \mathcal{N}_i |} \biggl( (P^i_{k|k})^{-1}x^{i}_{k|k} + \sum_{j \in \mathcal{N}_i} (P^j_{k|k})^{-1}x^{j}_{k|k} \biggr) \\
\bar{Y}^{i}_{k|k} &= \frac{1}{| \mathcal{N}_i |} \biggl( (P^i_{k|k})^{-1} + \sum_{j \in \mathcal{N}_i} (P^j_{k|k})^{-1}  \biggr)
\end{align}
This can be seen as a generalized form of covariance intersection \cite{julier1997non}. Censoring modifies the number of neighbors observed $\mathcal{N}_i$, maintaining consistency due to the $1 / |\mathcal{N}_i|$ weighting. Then, we must show consistency of the prediction step of the algorithm, given by
\begin{align*}
Y^{i}_{k+1|k} &= (I+M^i_k Q_k)^{-1} M^i_k \quad M^i_k &= (A_k^{-1})^T \bar{Y}^i_{k|k} A_k^{-1}.
\end{align*}
by using the properties  $(I+AB)^{-1}=I-A(I+BA)^{-1}B$ and $(I+A^{-1})^{-1}=A(A+I)^{-1}$, we rewrite the update as
\begin{align}
Y^{i}_{k+1|k} &= (A_k^{-1})^T \bar{Y}^i_{k|k} A_k^{-1} \nonumber\\
&-(A_k^{-1})^T \bar{Y}^i_{k|k} 
\left( \bar{Y}^i_{k|k}  +  A_k^T Q_k^{-1}  A_k\right) 
\bar{Y}^i_{k|k} A_k^{-1}.
\end{align}
Then using the function $f(Y)$, defined as
\begin{align*}
f(Y)\triangleq & (A^{-1})^T Y A^{-1} 
-(A^{-1})^T Y 
\left( Y  +  A^T Q^{-1}  A\right) 
Y A^{-1},
\end{align*}
with property $Y_1 \leq Y_2$ then $f(Y_1) \leq f(Y_2)$ \cite{horn2012matrix}, it can be shown that the filtering step is consistent.
\end{proof}

\end{document}